\documentclass[10pt,twocolumn,letterpaper]{article}

\usepackage{cvpr}              
\definecolor{cvprblue}{rgb}{0.21,0.49,0.74}
\usepackage[pagebackref,breaklinks,colorlinks,allcolors=cvprblue]{hyperref}

\usepackage{amsmath}
\usepackage{amssymb}
\usepackage{mathtools}
\usepackage{amsthm}

\usepackage{amsmath,amsfonts,bm}

\def\eqref#1{equation~\ref{#1}}

\def\1{\bm{1}}

\def\btheta{{\bm{\theta}}}

\def\bphi{{\bm{\phi}}}

\def\rvh{{\mathbf{h}}}

\def\rvv{{\mathbf{v}}}
\def\rvw{{\mathbf{w}}}
\def\rvx{{\mathbf{x}}}

\def\rvz{{\mathbf{z}}}

\def\rmA{{\mathbf{A}}}
\def\rmB{{\mathbf{B}}}

\def\rmI{{\mathbf{I}}}

\def\beps{{\bm{\epsilon}}}

\DeclareMathAlphabet{\mathsfit}{\encodingdefault}{\sfdefault}{m}{sl}
\SetMathAlphabet{\mathsfit}{bold}{\encodingdefault}{\sfdefault}{bx}{n}

\newcommand{\E}{\mathbb{E}}

\newcommand*\diff{\mathop{}\!\mathrm{d}}

\usepackage[dvipsnames]{xcolor}

\usepackage{pifont}
\usepackage[table]{xcolor}
\definecolor{OrangeBase}{RGB}{255,128,0}    
\colorlet{torange}{OrangeBase!20} 
\usepackage[capitalize,noabbrev]{cleveref}
\theoremstyle{plain}
\newtheorem{theorem}{Theorem}[section]
\newtheorem{proposition}[theorem]{Proposition}

\theoremstyle{definition}

\theoremstyle{remark}

\usepackage{comment}

\newcommand{\eqnmarkbox}[2]{\colorbox{#1}{$\displaystyle #2$}}

\def\paperID{15599} 
\def\confName{CVPR}
\def\confYear{2026}

\title{MeanFlow Transformers with Representation Autoencoders}

\author{
Zheyuan Hu\textsuperscript{1} \quad
Chieh-Hsin Lai\textsuperscript{1} \quad
Ge Wu\textsuperscript{3} \quad
Yuki Mitsufuji\textsuperscript{1,2} \quad
Stefano Ermon\textsuperscript{4} \\
\textsuperscript{1}Sony AI \quad
\textsuperscript{2}Sony Group Corporation \quad
\textsuperscript{3}Nankai University \quad
\textsuperscript{4}Stanford University \\
\href{mailto:zyhu2001@gmail.com}{\texttt{zyhu2001@gmail.com}} \quad
\href{mailto:chieh-hsin.lai@sony.com}{\texttt{chieh-hsin.lai@sony.com}}
}

\begin{document}
\maketitle

\begin{abstract}
MeanFlow (MF) is a diffusion-motivated generative model that enables efficient few-step generation by learning long jumps directly from noise to data. In practice, it is often used as a latent MF by leveraging the pre-trained Stable Diffusion variational autoencoder (SD-VAE) for high-dimensional data modeling. However, MF training remains computationally demanding and is often unstable. During inference, the SD-VAE decoder dominates the generation cost, and MF depends on complex guidance hyperparameters for class-conditional generation. In this work, we develop an efficient training and sampling scheme for MF in the latent space of a Representation Autoencoder (RAE), where a pre-trained vision encoder (e.g., DINO) provides semantically rich latents paired with a lightweight decoder. We observe that naive MF training in the RAE latent space suffers from severe gradient explosion. To stabilize and accelerate training, we adopt Consistency Mid-Training for trajectory-aware initialization and use a two-stage scheme: distillation from a pre-trained flow matching teacher to speed convergence and reduce variance, followed by an optional bootstrapping stage with a one-point velocity estimator to further reduce deviation from the oracle mean flow. This design removes the need for guidance, simplifies training configurations, and reduces computation in both training and sampling.  Empirically, our method achieves a 1-step FID of 2.03, outperforming vanilla MF’s 3.43, while reducing sampling GFLOPS by 38\% and total training cost by 83\% on ImageNet~256. We further scale our approach to ImageNet~512, achieving a competitive one-step FID of 3.23 with the lowest GFLOPS among all baselines.
Code is available at \url{https://github.com/sony/mf-rae}.
\end{abstract}

\begin{figure*}
\centering
\includegraphics[width=0.7\linewidth]{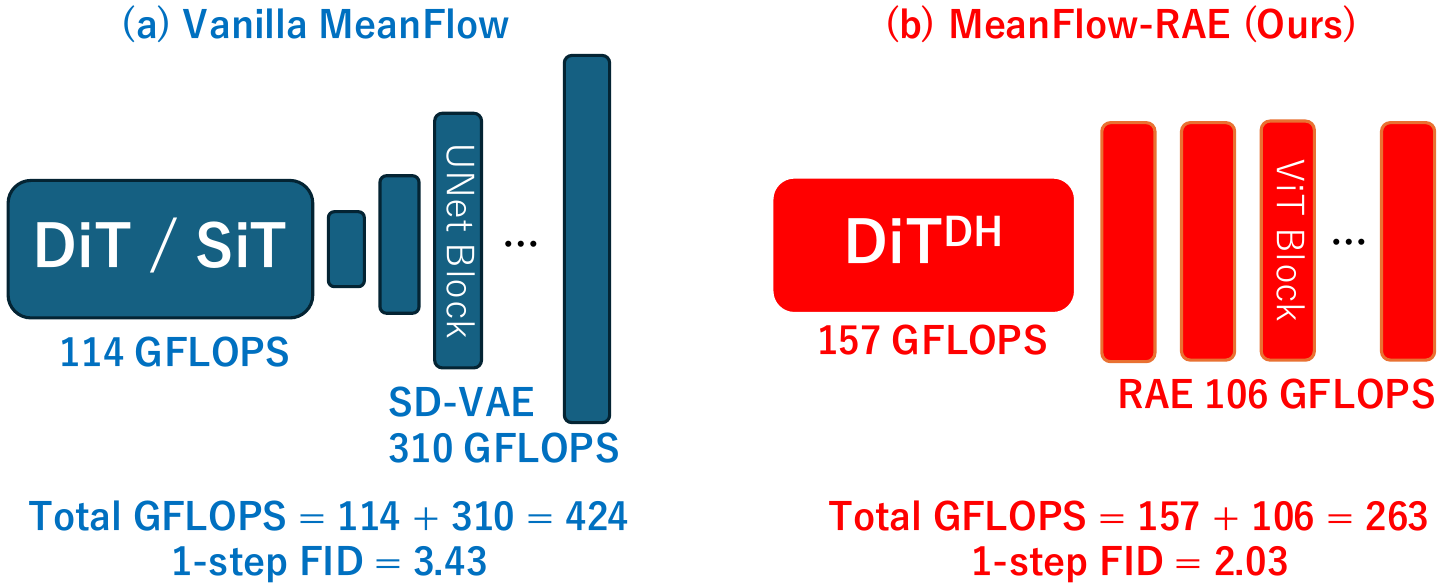}
\caption{\textbf{Overview of our method’s advantages.} On ImageNet~256, vanilla MF (a) employs the slow SD-VAE decoder, which accounts for 73\% of the total generation cost and thus bottlenecks the few-step generation speed. In contrast, our MF-RAE (b) leverages a higher-dimensional RAE latent space with semantically rich features and an efficient decoder. The DiT$^\text{DH}$ architecture is adopted to effectively process the high-dimensional latent space. As a result, while the converged 1-step FID of vanilla MF is 3.43 after more than 600 H100 GPU-days, our MF-RAE achieves a superior FID of 2.03 in only 100 H100 GPU-days. Additionally, our total generation cost is reduced by 38\% compared to vanilla MF in terms of GFLOPS, even with the same 1-step generation setting.}
\label{fig:method_overview}
\end{figure*}

\section{Introduction}

Diffusion models (or flow matching)~\citep{sohl2015deep,song2019generative,ho2020denoising,song2020score} 
have been shown to achieve high-fidelity sample generation. 
Its sampling can be interpreted as solving the associated probability flow ordinary differential equation (PF-ODE)~\citep{song2020score}. 
However, this procedure requires many neural network evaluations to approximate the numerical integration, 
which makes diffusion model generation notoriously slow.


Recent research~\citep{song2023consistency,kim2023ctm,geng2025mean,lai2025principlesdiffusionmodels}
has shifted toward \emph{flow map models}, with a promising representative called \emph{MeanFlow} (MF)~\citep{geng2025mean}. 
These models directly learn the PF-ODE solution map, transporting any initial state at one time to the corresponding state on the same trajectory at another time.
As a result, they enable generation in a few steps, mapping pure noise to clean data with only a small number of network evaluations. In practice, MF is often used in a latent space by leveraging a pre-trained Stable Diffusion variational autoencoder (SD-VAE)~\citep{rombach2022high} for high-dimensional image generation.
Despite these advances in diffusion-based few-step generative models, MF training and inference remain inefficient for such high-dimensional latent representations.

Training MF, even in a latent space, still requires hundreds of H100 GPU-days~\citep{meanflow_pytorch} on high-dimensional datasets such as ImageNet~256~\citep{deng2009imagenet}. 
It is further complicated by intricate classifier-free guidance (CFG)~\citep{ho2021classifier} configurations for class-conditional generation, which involve two CFG scale hyperparameters and two additional hyperparameters controlling the CFG triggering interval. 
These hyperparameters must be carefully tuned through extensive grid search to maximize MF performance, thereby increasing the overall complexity of training.
Moreover, the Jacobian vector product (JVP) required by the MF loss introduces an additional source of computational cost and instability.
Even when it is computed using the most efficient forward mode automatic differentiation~\citep{shi2024stde}, the JVP remains a significant bottleneck during training.
Supporting JVP in modern components such as Flash Attention~\citep{dao2022flashattention} also requires extra implementation effort~\citep{zheng2025large}, making MF cumbersome and time-consuming to adapt to new model architectures.

Regarding the inference, although MF enables few-step generation, the computational cost of the SD-VAE decoder that maps generated latent vectors back to pixel space dominates and substantially slows down the overall generation speed~\citep{kim2024pagoda}.
Specifically, as shown in \Cref{fig:method_overview} (a), for vanilla MF with a conventional DiT/SiT architecture combined with SD-VAE on ImageNet~256, the DiT/SiT requires 114 GFLOPS, whereas the SD-VAE decoder consumes 310 GFLOPS, which indicates that approximately 73\% of the total computation is spent on the decoder.

Recent research on the Representation Autoencoder (RAE)~\citep{zheng2025rae} 
replaces the conventional SD-VAE in latent diffusion models with a frozen pre-trained representation encoder (e.g., DINO~\citep{caron2021emerging}) and trains only a ViT-based~\citep{dosovitskiy2020image} decoder on top of its latent tokens.
Unlike the classic SD-VAE that uses a U-Net backbone to compress images into a low-dimensional latent space, 
RAE modernizes this design by adopting a transformer architecture and using semantically rich, high-dimensional representations as the generative space.
To model these higher-dimensional latents, RAE augments a DiT backbone with a wide yet lightweight DDT head~\citep{ddt} tailored for latent diffusion modeling, yielding an expressive and efficient DiT$^{\text{DH}}$ architecture.

For latent diffusion models, RAE brings limited improvement in sampling speed, since the main bottleneck is the large number of network evaluations required to solve the latent PF-ODE.
However, we emphasize that RAE’s decoder efficiency is particularly crucial for few-step models such as MF: its decoder requires about 106 GFLOPS, nearly a $3\times$ reduction compared to the $\sim$310 GFLOPS SD-VAE decoder used in vanilla MF.
This directly alleviates the decoder side bottleneck in current few-step generation and is a key motivation for our approach.

To this end, we improve the efficiency of MF training and sampling by learning
MF in the RAE latent space, and show that our approach is stable, fast, and high-quality in practice. We systematically analyze and decompose MF training into the following components.

First, we observe that naively training MF with the DiT$^{\text{DH}}$ architecture in the RAE latent space is unstable: gradients explode early in training, regardless of whether MF is initialized randomly or from a pre-trained flow matching teacher. We attribute this to a mismatch between the training signal of flow matching and the objective of MF: flow matching learns infinitesimal transitions along the PF-ODE trajectory, whereas MF must learn long jumps between distant time steps, and random initialization further aggravates this issue. To address this, we initialize MF with weights obtained from Consistency Mid-Training (CMT)~\citep{hu2025cmt}, which learns a trajectory-aware initialization by following the numerical PF-ODE trajectory of a pre-trained flow matching model.

Second, after stabilizing MF training with CMT, we adopt the MeanFlow Distillation (MFD) algorithm, which efficiently converts a pre-trained flow matching model (teacher) into a few-step MF model. We then introduce a novel optional bootstrapping stage that replaces the teacher with a one-point velocity estimator and performs a brief, low-cost fine-tuning phase. This bootstrapping stage becomes crucial for decoupling MF performance from the teacher’s quality, particularly when the teacher is suboptimal: we show theoretically and verify empirically that this two-stage procedure breaks the performance ceiling that the original teacher model would otherwise impose on MF.

To avoid ad-hoc tuning of guidance hyperparameters (e.g., the CFG scale and effective intervals, or the Auto-Guidance strength plus an additional guidance model~\citep{karras2024edm2}) in class-conditional generation, which otherwise makes MF highly sensitive to configuration, we distill a pre-trained class-conditional flow matching model in the RAE latent space into an MF model in the same space. This yields a class-conditional MF model that operates entirely without guidance parameters. Removing guidance therefore both simplifies the configuration and reduces the computational cost per iteration, since guided MF requires extra model evaluations, leading to slower convergence in practice.

Third, to further accelerate MF training, we replace the Jacobian--vector
product (JVP) term in the MF loss with a finite-difference approximation, which achieves similar empirical performance to exact JVP~\citep{wang2025transition}.

We refer to the resulting model, equipped with all these components, as \emph{MeanFlow-RAE} (MF--RAE). 
These changes stabilize optimization and significantly speed up convergence, as we will demonstrate empirically. 
Moreover, we empirically observe that MF-RAE can largely reuse the hyperparameters from flow matching pre-training, with only minor modifications. 
We attribute this robustness to the expressive RAE latent space in which MF-RAE operates. 
In contrast, vanilla MF with SD-VAE latents requires careful retuning.
Finally, sampling with MF--RAE is also accelerated thanks to the lightweight RAE decoder and the few-step generation nature of MF.

We validate our approach on ImageNet. 
At 256$\times$256 resolution (see \Cref{fig:method_overview} for an overview), 
we achieve an FID of 2.03 with a single sampling step using approximately 100 H100 GPU-days of training in total
(including flow matching pre-training, CMT mid-training, and MF post-training),
compared to vanilla MF, which attains an FID of 3.43 after more than 600 H100 GPU-days. 
Under the same single-step generation setting, our method also reduces the total GFLOPS by 38\%, thanks to the efficient RAE decoder. 
This delivers higher image quality and faster generation with substantially lower training cost.
We further scale up MF-RAE to ImageNet~512 and achieve a competitive 1-step FID score of 3.23 while maintaining the lowest sampling GFLOPS cost.

Overall, the MF-RAE framework advances few-step flow map models along three axes:
it reduces configuration complexity by removing guidance hyperparameters, 
stabilizes optimization via CMT-based initialization and MFD-based training targets, 
and accelerates both training and sampling while improving sample quality.
Because it is built on a generic RAE latent space with a DiT-based backbone,
MF-RAE remains compatible with future improvements in training algorithms and transformer-based architectures for flow map models,
thereby providing an extensible and general pipeline for efficient few-step generation.

\section{Preliminary}
\paragraph{Representation Autoencoder (RAE).} RAE~\citep{zheng2025rae} replaces the conventional SD-VAE of dimensionality compression with a pre-trained semantic representation encoder $E$, such as DINOv2~\citep{oquab2023dinov2} or SigLIP2~\citep{tschannen2025siglip}. The encoder $E$ is kept frozen, while a ViT-based \cite{dosovitskiy2020image} decoder $D$ is trained to achieve high-fidelity reconstruction by leveraging high-dimensional latent representations.
Specifically, given an input image $\rvx \in \mathbb{R}^{3\times H\times W}$, the frozen encoder $E$ extracts a semantic representation $\rvz_0:=E(\rvx)$, which is subsequently decoded by $D$ to reconstruct the pixel-level image $\hat{\rvx}:=D(\rvz_0)$.
The decoder is optimized using a composite objective $\mathcal{L}_{\operatorname{rec}}$ that combines $\operatorname{L_1}$, learned perceptual (LPIPS)~\citep{zhang2018unreasonable}, and adversarial  losses (GAN)~\citep{goodfellow2014generative}:
\begin{align*}
 \mathcal{L}_{\operatorname{rec}}(\rvx) = \omega_{\text{L}} \operatorname{LPIPS}(\hat{\rvx}, \rvx) + &\operatorname{L_1}(\hat{\rvx}, \rvx) + \omega_{\text{G}} \eta \operatorname{GAN}(\hat{\rvx}, \rvx),
\end{align*}
which defines a high-quality reconstruction objective, where $\omega_{\text{L}}$, $\omega_{\text{G}}$, and $\eta$ denote the weights of the respective loss terms. In addition, RAE extends DiT to DiT$^{\text{DH}}$ by incorporating a wide yet efficient DDT head~\citep{ddt}, enabling effective modeling of the high-dimensional latent space.

\paragraph{Flow Matching (FM) and Diffusion Model.} FM (or diffusion model)~\citep{lipman2022flow,ma2024sit} interpolates between clean data (or latent representations $\rvz_0=E(\rvx)$) $\rvz_0 \sim p_{\mathrm{data}}$ with noise $\beps\sim\mathcal{N}(\mathbf{0},\mathbf{I})$ via $\rvz_t=\alpha_t \rvz_0+\sigma_t\beps$, where a typical choice is $\alpha_t=1-t$, $\sigma_t=t$ for $t\in[0,1]$. Data generation from noise is achieved by learning a vector field $\rvv_{\bphi}(\rvz_t,t)$ to match the conditional velocity $\alpha_t' \rvz_0 + \sigma_t' \beps$ on average:
\[
\mathcal{L}_{\mathrm{FM}}(\btheta)
=\E_t\,\E_{\rvz_0,\beps}\!\left[w(t)\big\|\rvv_{\bphi}(\rvz_t,t)-\big(\alpha_t' \rvz_0 + \sigma_t' \beps\big)\big\|_2^2\right],
\]
where $w(t)$ is a time weighting function.
The optimum is achieved as $\rvv(\rvz_t,t)=\E[\alpha_t'\rvz_0+\sigma_t'\beps| \rvz_t]$.

Given pre-trained $\rvv_{\bphi}(\rvz_t,t)\approx \rvv(\rvz_t, t)$, generation is achieved via the integration of PF-ODE~\citep{song2020score}, $\frac{\mathrm{d}\rvz_t}{\mathrm{d}t}=\rvv(\rvz_t,t)$, starting from $\rvz_1\sim \mathcal{N}(\mathbf{0},\mathbf{I})$ down to $t=0$.
However, solving the PF-ODE requires tens or even hundreds of model inferences, since 
$\rvv$ captures only infinitesimal PF-ODE transitions. This makes FM's generation computationally expensive due to numerical discretization. MF addresses this limitation by directly learning long ODE jumps, enabling few-step generation.

\paragraph{MeanFlow (MF).} MF~\citep{geng2025mean} learns the average velocity integration for few-step generation, stemming from the idea of fitting arbitrary long PF-ODE jumps between $t, s$~\citep{kim2023ctm}:
\begin{align*}
\rvh_{\btheta}(\rvz_t, t, s)  \approx \rvh(\rvz_t, t, s) = \frac{1}{t-s}\int_s^t \rvv(\rvz_u,u)\diff u.
\end{align*}
Taking the derivative with $t$ to obtain the MF identity provides a tractable optimization target:
\begin{align}\label{eq:mf_loss}
\mathcal{L}_{\mathrm{MF}}(\btheta) := \mathbb{E}_{t > s} \mathbb{E}_{\rvz_t}\Big[\|\rvh_{\btheta}(\rvz_t,t,s) - \rvh_{\btheta^-}^{\mathrm{tgt}}(\rvz_t,t,s)\|_2^2\Big],
\end{align}
where the stop-grad regression target (with stop-gradient parameters $\btheta^-$) is defined as
\[
\rvh_{\btheta^-}^{\mathrm{tgt}}(\rvz_t,t,s) := \rvv(\rvz_t,t) - (t-s)\bigl(\rvv(\rvz_t,t)\partial_\rvx \rvh_{\btheta^-} + \partial_t \rvh_{\btheta^-}\bigr).
\]
The ground truth velocity $\rvv(\rvz_t,t)$ is either replaced by (1) one-point estimation $\alpha_t' \rvz_0 + \sigma_t' \beps$ where $\rvz_t=\alpha_t \rvz_0+\sigma_t\beps$ in MF Training (MFT) or (2) a teacher flow matching model in MF Distillation (MFD).

\section{Proposed Method: MF-RAE}

We revisit the MF training loss defined in \Cref{eq:mf_loss} in a more general form, which clarifies the MF training design and suggests a principled recipe for designing flow map models more broadly.

Let $\rvw$ be a vector, and consider the stop-gradient regression target defined as
\begin{align*}
\rvh_{\btheta^-}^{\mathrm{tgt}}(\rvz_t,t,s;\rvw)
:= \rvw - (t-s)\bigl(\eqnmarkbox{pink!30}{(\partial_\rvz \rvh_{\btheta^-})\rvw + \partial_t \rvh_{\btheta^-}}\bigr),
\end{align*}
where $\rvw$ is either chosen as
a one-point estimate of the conditional velocity
$\hat{\rvv}(\rvz_t, t) := \alpha_t' \rvz_0 + \sigma_t' \beps$ in MFT,
or as the output of a pre-trained flow matching model $\rvv_{\bphi}(\rvz_t, t)$ in MFD. Among $\rvh_{\btheta^-}^{\mathrm{tgt}}$,  the transport derivative  $(\partial_\rvz \rvh_{\btheta^-})\,\rvw + \partial_t \rvh_{\btheta^-}$ along $\rvw$
can be computed as a JVP of $\rvh_{\btheta^-}$ with respect to $(\rvz,t,s)$
in the direction $[\rvw,  1,  0]$, i.e.
\[
[\partial_\rvz \rvh_{\btheta^-},\,\partial_t \rvh_{\btheta^-},\,\partial_s \rvh_{\btheta^-}]^\top
[\rvw,  1,  0].
\]

This general target $\rvh_{\btheta^-}^{\mathrm{tgt}}(\rvz_t,t,s;\rvw)$
induces a generalized MF loss, which we denote by
\begin{align}\label{eq:general_mf_loss}
\mathcal{L}_{\mathrm{MF}}(\btheta;\rvw)
:= \mathbb{E}_{t > s}\,\mathbb{E}_{\rvz_t}\Big[
   \|\eqnmarkbox{blue!20}{\rvh_{\btheta}}(\eqnmarkbox{orange!20}{\rvz_t},t,s)
    - \rvh_{\btheta^-}^{\mathrm{tgt}}(\rvz_t,t,s;\eqnmarkbox{green!20}{\rvw})\|_2^2
\Big].
\end{align}
Building on this general formulation, we develop a systematic view of MF training and propose concrete improvements along four key axes:
\begin{itemize}
    \item \Cref{subsec:rae-latent}: a better latent space $\eqnmarkbox{orange!20}{\rvz_t}$ for MF modeling using RAE latents;
    \item \Cref{subsec:cmt-initial}: trajectory-aware initialization for the MF $\eqnmarkbox{blue!20}{\rvh_{\btheta}}$ via Consistency Mid-Training (CMT);
    \item \Cref{subsec:mfd-mft}: choice of the proxy velocity $\eqnmarkbox{green!20}{\rvw}$ in $\rvh_{\btheta^-}^{\mathrm{tgt}}(\cdot;\rvw)$ through a trade-off between MFD and MFT;
    \item \Cref{subsec:finite-difference}: efficient  computation of the transport derivative $\eqnmarkbox{pink!30}{(\partial_\rvz \rvh_{\btheta^-})\rvw + \partial_t \rvh_{\btheta^-}}$ via finite differences.
\end{itemize}

We refer to the resulting designs of MF with RAE-based latent modeling and our tailored training scheme as \emph{MeanFlow-RAE} (MF-RAE). In the following sections, we present the components of MF-RAE in detail.

\subsection{MF with DiT$^{\text{DH}}$ Architecture and RAE Latents}\label{subsec:rae-latent}
RAE, originally proposed for latent flow matching (diffusion) models, achieves strong generation quality even \emph{without} guidance by leveraging the expressive latent space of a pre-trained visual encoder. However, in standard latent diffusion models, RAE brings only limited gains in wall-clock sampling speed: although its ViT-based decoder is lighter in terms of GFLOPS, the dominant bottleneck remains the iterative PF-ODE solving of the latent diffusion model itself.

In contrast, RAE is particularly well-suited to MF and, more broadly, to few-step latent flow map models. Since MF evaluates the latent model in only one or two steps, the overall generation cost is no longer dominated by iterative dynamics. Thus, RAE's efficient decoder's acceleration is more pronounced. Moreover, the rich RAE latents can accelerate MF convergence. They also enable high-quality generation \emph{without} any guidance such as CFG or Auto-Guidance, thereby eliminating guidance-related hyperparameters and significantly simplifying MF training.

We propose the MF-RAE transformer architecture by extending the DiT$^{\text{DH}}$ backbone used in RAE with an additional time-embedding module for the time difference $t-s$. Specifically, we sum the embeddings of the class label, the current time $t$, and the time difference $t-s$, whereas the original DiT$^{\text{DH}}$ only sums the label and time-$t$ embeddings. This simple change allows the model to explicitly encode both the absolute time and the time difference, which is important for learning accurate flow maps in MF.

This modification is analogous to the change used in vanilla MF when adapting DiT/SiT backbones of the diffusion model for flow map learning, ensuring a fair comparison with the vanilla MF baseline. More broadly, any architectural changes developed to adapt DiT/SiT to MF can be incorporated in the same way into the DiT$^{\text{DH}}$ backbone. Thus, future advances in  architectures for MF can be plugged into our MF-RAE without altering the overall design, highlighting our generality and extensibility.

\subsection{Stabilizing MF via Consistency Mid-Training (CMT) Initialization}\label{subsec:cmt-initial}

Although DiT$^{\text{DH}}$ with RAE is a stronger and more efficient backbone than DiT/SiT with SD-VAE for latent modeling, we find that naively training MF $\rvh_\btheta$ on top of RAE latents is highly unstable: the gradients explode when the model $\rvh_\btheta$ is initialized either randomly or from a pre-trained teacher FM model. Empirically, optimization with the XL-sized model diverges almost immediately. At smaller scales (S and B), both random and diffusion-based initializations remain stable only in the very early phase; the loss then gradually increases and eventually blows up. The best 1-step FID observed before divergence is still above 20, which is far from convergence.

To mitigate this instability, we initialize $\rvh_{\btheta}$ using weights obtained from Consistency Mid-Training (CMT)~\citep{hu2025cmt}. Instead of starting the MF model $\rvh_\btheta$ from a pre-trained infinitesimal flow matching model (which learns only local jumps along the PF-ODE trajectory) or from an unstructured random initialization, we first run CMT to learn a trajectory-aware initialization from the numerical ODE trajectory of the teacher flow matching model:
\begin{align}\label{eq:cmt-mf}
\begin{aligned}
    \mathcal{L}_{\text{CMT-MF}}(\btheta)
=\mathbb{E}_{i>j}  \mathbb{E}_{\rvz_T\sim p_{\mathrm{prior}}}
 \Bigl[\bigl\|\rvh_{\btheta}(\hat\rvz_{t_i}, t_i, t_j) 
 - \tfrac{\hat\rvz_{t_i} - \hat\rvz_{t_j}}{t_i - t_j}\bigr\|_2^2\Bigr],
\end{aligned}
\end{align}
where $\{\hat{\rvz}_{t_i}\}$ is the teacher FM model's trajectory, obtained by integrating from a prior sample $\rvz_T \sim p_{\mathrm{prior}}$ and evaluating it at the discrete time grid $\{t_i\}$. In other words, CMT warm-up trains $\rvh_{\btheta}$ to reproduce a proxy of the long jumps required by MF by matching the corresponding long transitions along the teacher trajectory.

In the original CMT setup, a high-order multistep ODE solver (e.g., second-order Heun) is used to obtain accurate teacher trajectories within 16 NFEs.
In our RAE setting, a simple first-order Euler solver with 16 NFEs already
suffices: on ImageNet~256, the RAE diffusion achieves FID 1.51 with 50 steps and 2.32 with 16 steps, which is more than adequate for CMT's  teacher.

\subsection{Trade-Offs Between MFD and MFT}\label{subsec:mfd-mft}

With the formulation in \Cref{eq:general_mf_loss}, $\rvw$ is usually chosen either as
a point estimate of the conditional velocity
$\hat{\rvv}$ (MFT),
or as the output of a pre-trained diffusion teacher $\rvv_{\bphi}$ (MFD).
However, it remains unclear, in a principled sense, which choice of $\rvw$ is more beneficial for the training.

The following proposition clarifies this by characterizing how replacing the oracle MF
$\rvh(\rvz_t, t, s) = \tfrac{1}{t-s}\int_s^t \rvv(\rvz_u,u),\diff u$
with the proxy $\rvh_{\btheta^-}^{\mathrm{tgt}}(\rvz_t,t,s;\rvw)$, together with the choices $\rvw = \hat{\rvv}$  or $\rvw = \rvv_{\bphi}$,
makes the practical objective $\mathcal{L}_{\mathrm{MF}}(\btheta;\rvw)$ deviate from the oracle objective loss function.

\begin{proposition}
For any $\lambda \in [0,1]$, consider the combination of the one-point estimator
and the pre-trained velocity
\[
\rvw_\lambda := (1-\lambda)\hat{\rvv} + \lambda\rvv_{\bphi}.
\]
Plugging $\rvw_\lambda$ into the target
$\rvh^{\mathrm{tgt}}_{\btheta^-}(\rvz_t,t,s;\rvw)$ (with $\rvw = \rvw_\lambda$)
yields the corresponding loss $\mathcal{L}_{\mathrm{MF}}(\btheta;\rvw_\lambda)$. Consider the following three residuals: the one-point velocity residual, the teacher–oracle velocity residual, and the oracle bias.
\begin{align*}
\delta\hat{\rvv}_t &:= \hat{\rvv}(\rvz_t,t) - \rvv(\rvz_t,t),\\
\delta\rvv_t^\bphi &:= \rvv_{\bphi}(\rvz_t,t) - \rvv(\rvz_t,t),\\
\delta\rvh(\rvz_t,t,s) &:= \rvh_{\btheta^-}(\rvz_t,t,s) - \rvh(\rvz_t,t,s).
\end{align*}
Then the MF loss admits the following decomposition:
\begin{align}\label{eq:LMF-mix}
\begin{aligned}
    \mathcal{L}_{\mathrm{MF}}(\btheta;\rvw_\lambda)
    &= \mathbb{E}_{t,\rvz_t}
       \Big\|\rvh_{\btheta}
       - \big(\rvh + \rmB + \lambda\,\rmA_{\btheta^-}\,\delta\rvv_t^\bphi\big)\Big\|^2 \\
    &\qquad\qquad + (1-\lambda)^2\,\mathbb{E}\big\|\rmA_{\btheta^-}\,\delta\hat{\rvv}_t\big\|^2,
\end{aligned}
\end{align}
where $\rmB(\rvz_t,t,s)
:= (t-s)\Big(\partial_t \delta\rvh
  + \big(\nabla_{\rvx}\delta\rvh\big)\rvv\Big)$ and  $\rmA_{\btheta^-}(\rvz_t,t,s)
:= \rmI - (t-s)\,\nabla_{\rvx}\rvh_{\btheta^-}$.
\label{prop:bias-var}
\end{proposition}
The proof is provided in \Cref{app:theory-proof}. When $\lambda = 0$, we have $\rvw_0 = \hat{\rvv}$, and \Cref{eq:LMF-mix} reduces to
\[
\mathcal{L}_{\mathrm{MF}}(\btheta;\rvw_0)
= \mathbb{E}_{t,\rvz_t}\big\|\rvh_{\btheta} - (\rvh + \rmB)\big\|^2
  + \mathbb{E}\big\|\rmA_{\btheta^-}\,\delta\hat{\rvv}_t\big\|^2,
\]
which corresponds to MF trained purely from the one-point velocity estimator (MFT). In this case, the loss includes the full variance induced by the noisy one point estimate $\hat{\rvv}$ through the second term. When $\lambda = 1$, we have $\rvw_1 = \rvv_{\bphi}$, and the variance term disappears:
\[
\mathcal{L}_{\mathrm{MF}}(\btheta;\rvw_1)
= \mathbb{E}_{t>s,\rvz_t}\big\|\rvh_{\btheta} - (\rvh + \rmB + \rmA_{\btheta^-}\,\delta\rvv_t^\bphi)\big\|^2.
\]
This corresponds to the pure distillation regime (MFD), where the objective depends only on the teacher velocity residual $\delta\rvv_t^\bphi$ and the oracle bias $\rmB$.

Therefore, when the teacher model is of sufficiently high quality, such as a strong RAE flow matching model (i.e., $\delta\rvv_t^\bphi \approx \bm{0}$), MFD yields both smaller bias and lower variance, leading to a faster convergence rate.  In practice, however, MFD is still limited by the quality of the teacher, since it inevitably inherits some bias whenever $\delta\rvv_t^\bphi \neq \bm{0}$.
To further reduce this residual bias, we may optionally apply MFT after MFD has converged.
Although MFT typically has higher variance, this variance is effectively reduced when starting from a well converged MFD model, allowing us to benefit from the smaller bias characteristic of MFT. When the teacher model is sufficiently strong, the MF model obtained by MFD alone already achieves good performance, and the additional MFT stage becomes unnecessary.

To the best of our knowledge, we are the first to provide a clear theoretical analysis of the trade-off between MFD and MFT. In summary, our theoretical results suggest a practical bias–variance control procedure that first applies MFD with a pre-trained flow matching model, followed by an optional MFT stage using a one-point estimate.

\subsection{Finite Difference for Generality}\label{subsec:finite-difference}
The main computational and stability bottleneck of MF is the JVP required in the transport derivative $\tfrac{\diff}{\diff t} \rvh_{\btheta^-}(\rvz_t, t, s)= (\partial_\rvz \rvh_{\btheta^-})\rvw + \partial_t \rvh_{\btheta^-}$  in its regression target.
A natural way to avoid explicit JVPs is to approximate the time derivative with a finite difference~\citep{wang2025transition}.
Given a step size $\Delta t$, we write
\begin{align*}
    \tfrac{\diff}{\diff t}\rvh_{\btheta}(\rvz_{t}, t, s)
    &\approx 
    \tfrac{\rvh_{\btheta}(\rvz_{t + \Delta t}, t + \Delta t, s)
          - \rvh_{\btheta}(\rvz_{t - \Delta t}, t - \Delta t, s)}{2\Delta t},
\end{align*}
where
\[
\rvz_{t \pm \Delta t} \approx \rvz_{t} \pm \Delta t \rvw(\rvz_t,t)
\]
is obtained by a first-order Euler step along the teacher velocity field $\rvw = \rvv_\bphi$.

Empirically, choosing $\Delta t \in [0.001, 0.01]$ yields stable training, and values in this range lead to similar convergence behavior and performance comparable to using exact JVPs (i.e., the limit $\Delta t \to 0$), indicating that the discretization error is negligible in practice~\citep{wang2025transition}.
Hence, we simply fix the middle value $\Delta t = 0.005$ throughout our experiments.
%

\subsection{Summary of MF-RAE Pipeline}\label{subsec:summary-alg}
In summary, we decompose the challenging MF-RAE training into three more manageable stages, adopting a divide-and-conquer approach. Each stage plays a crucial role in facilitating and stabilizing the subsequent stage.
\begin{enumerate}
\item[(1)] \textbf{Pre-training:} Train a high-quality flow matching teacher in the RAE latent space.
\item[(2)] \textbf{Mid-training:} Apply CMT to learn a trajectory-aware MF initialization (using the proposed DiT$^{\text{DH}}$ architecture), where the pre-trained teacher generates the reference trajectories and serves as CMT's initialization.
\item[(3)] \textbf{Post-training:} Starting from the CMT weights, train the MF model in the RAE latents using the MFD with finite differences; optionally, apply MFT afterward to further reduce loss bias and improve model quality.
\end{enumerate}

\section{Related Work}
\subsection{Diffusion Models and Flow Matching}
Diffusion models and flow matching both aim to learn a time-dependent velocity field that gradually transforms a simple Gaussian prior distribution  into the target data distribution. In diffusion models, this transformation can be described  by an deterministic PF-ODE, allowing for flexible generative sampling via the numerical simulation of ODEs~\citep{song2020score,lipman2022flow,rectified_flow}.

In high-dimensional image synthesis, Diffusion Transformers (DiT) in the
SD-VAE latent space pioneered scaling diffusion models to large,
high-fidelity tasks~\citep{rombach2022high,peebles2023dit}, enabling efficient training and inference while preserving semantic and visual quality. SiT~\citep{ma2024sit} extends DiT with flow matching interpolants for more flexible distribution transport. Subsequent work leverages semantic representations to improve reconstruction and generation~\citep{kouzelis2025boosting,chen2025masked,huang2025ming}, including REPA~\citep{yu2025repa}, which aligns SiT features with pre-trained encoders to speed up training, and REG~\citep{wu2025reg}, which further injects a pre-trained encoder’s class token in denoising to capture image-label pair information better.
RAE~\citep{zheng2025rae} further treats a discriminative encoder as a tokenizer to enable semantic-space reconstruction and generation while alleviating issues of high-dimensional latent spaces. However, these models still struggle to achieve high fidelity under few-step
sampling.

\subsection{Few-Step Flow Map Models}
Diffusion models and flow matching suffer from slow sampling due to the recursive model inferences required during ODE solving after discretizing numerous time steps. Flow map models, such as the consistency model (CM)~\citep{song2023cm}, consistency trajectory model (CTM)~\citep{kim2023ctm}, and MF~\citep{geng2025mean}, directly learn the solution map of a deterministic PF-ODE, thereby enabling fast few-step sampling. Specifically, CM learns to map any noisy point along the trajectory directly to its corresponding clean point on the same ODE trajectory. CTM extends CM by learning mappings between arbitrary points along the ODE trajectory. MF shares an mathematically equivalent parameterization with CTM but instead learns the average ODE integration between two points.

\section{Experiments}

\paragraph{Dataset and Setup.} We primarily evaluate sample quality using FID~\citep{heusel2017fid} on class-conditional ImageNet~256 and 512 benchmarks. To quantify training cost, we train all models on H100 GPUs and report the total training GPU time. To assess generation efficiency, we additionally report the compute cost in GFLOPS per generated sample.
Given a pre-trained RAE encoder–decoder pair, the MF-RAE training pipeline consists of three stages as shown in \Cref{subsec:summary-alg}. We defer the detailed configurations of the flow matching pre-training and CMT mid-training stages to Appendix \ref{appendix:hyperparams}. Below, we describe the simplified hyperparameter setup used for training the MF model in MF-RAE.

\paragraph{Hyperparameter Simplicity of MF-RAE.} 
We train MF-RAE with (almost) the same hyperparameters as the DiT$^{\text{DH}}$ flow matching stage, changing only a few scalars: we reduce the batch size from 1024 to 256/128 for ImageNet~256/512, lower the learning rate from $2\times 10^{-4}$ to $1\times 10^{-4}$ (and keep it fixed for both resolutions), and adjust the EMA rate from $0.9995$ to $0.9999/0.9995$ for ImageNet~256/512. The smaller batch sizes are purely for efficiency, enabled by the stable CMT initialization, while the learning rate and EMA are aligned with the original MF settings~\citep{geng2025mean} to isolate architectural and algorithmic effects. In practice, this means one can almost directly reuse the flow matching configuration and obtain a few-step MF-RAE generator with only minimal tweaks.

By contrast, vanilla MF requires substantial hyperparameter redesign relative to its flow matching teacher. The corresponding SiT + SD-VAE model~\citep{ma2024sit} is trained with uniform time sampling, whereas vanilla MF switches to a carefully tuned log-normal time distribution and further depends on delicate choices of CFG weights, CFG time intervals, and an additional CFG mixing scale $\kappa$ (in their notation) to make the method work well. MF-RAE needs none of these bespoke techniques: we keep the teacher’s uniform time sampling and use no guidance for class-conditional generation, yet still obtain fast, stable convergence. This highlights that MF-RAE is significantly more hyperparameter-robust and easier to deploy than vanilla MF.

\subsection{ImageNet~256 Main Results (\Cref{tab:imagenet-256})}
The sample quality results of MF-RAE compared with various baseline models are presented in \Cref{tab:imagenet-256}. Our MF-RAE achieves state-of-the-art (SOTA) 1-step and 2-step generation quality among all few-step flow map models. Furthermore, these strong results are obtained with both lower generation and training costs, as explained below.
\begin{table}[htbp]
\centering
\caption{\small{Sample quality on class-conditional ImageNet~256.}}
\label{tab:imagenet-256}
\fontsize{8.9pt}{9.9pt}\selectfont
\begin{tabular}{l@{\hskip 6pt}c@{\hskip 4pt}c@{\hskip 4pt}c}
\toprule
\textbf{METHOD} & \textbf{NFE} ($\downarrow$) & \textbf{FID} ($\downarrow$)  & \textbf{\#Params} \\
\midrule
\multicolumn{3}{l}{\textbf{Diffusion Models \& Flow Matching} (*no guidance)} \\
\midrule
ADM-G~\citep{diffusion_beat_gan} & 250$\times$2 & 3.94 & 554M \\
DiT-XL/2~\citep{peebles2023dit} & 250$\times$2 & 2.27  & 675M \\
SiT-XL/2~\citep{ma2024sit} & 250$\times$2 & 2.06  & 675M \\
REPA~\citep{yu2025repa} & 250$\times$2 & 1.29  & 675M \\
REG~\citep{wu2025reg} & 250$\times$2 & 1.36  & 675M \\
RAE$^*$~\citep{zheng2025rae} & 50 & 1.51  & 839M \\
RAE~\citep{zheng2025rae} & 50$\times$2 & \textbf{1.13}  & 839M \\
\midrule
\multicolumn{3}{l}{\textbf{GANs \& Masked Models}} \\ 
\midrule
BigGAN~\citep{brock2018biggan} & 1 & 6.95 & 112M \\
StyleGAN~\citep{sauer2022stylegan} & 1 & 2.30 & 166M \\
MAR~\citep{li2024MAR} & 256$\times$2 & \textbf{1.55}  & 943M \\
VAR-$d$30~\citep{tian2024VAR} & 10$\times$2 & 1.92  & 2B \\
\midrule
\multicolumn{3}{l}{\textbf{Flow Map Models}} \\
\midrule 
iCT~\citep{song2023ict} & 1 / 2 & 34.24 / 20.30 & 675M  \\
IMM~\citep{zhou2025inductive}& 2 & 7.77 & 675M  \\
Shortcut~\citep{frans2025shortcut}& 1 & 10.60 & 675M  \\
MeanFlow~\citep{geng2025mean} & 1 / 2 & 3.43 / 2.20 & 676M  \\
CMT w/ MF~\citep{hu2025cmt} & 1 & 3.34 & 676M  \\
AlphaFlow~\citep{zhang2025alphaflow} & 1 / 2 & 2.58 / 1.95 & 675M  \\
 \rowcolor{gray!20}
MF-RAE~(Ours) & 1 / 2 & \textbf{2.03} / \textbf{1.89} & 841M  \\
\bottomrule
\end{tabular}
\end{table}

\paragraph{Faster Generation of MF-RAE over Baselines.} The total generation cost, measured in GFLOPS, is the decoder cost plus (NFE) times the diffusion transformer cost. 
In the ImageNet~256 experiments, our method achieves lower GFLOPS and thus faster generation speed for both 1-step and 2-step generation.
For the 1-step case, vanilla DiT-based MF requires $310+114=424$ GFLOPS, whereas our approach uses only $106+157=263$ GFLOPS.
For the 2-step case, vanilla DiT-based MF costs $310+114 \times 2=538$ GFLOPS, while ours requires just $106+157 \times 2=420$ GFLOPS.
Despite having the same NFE, MF-RAE enables faster generation, as illustrated in \Cref{fig:method_overview}.

\paragraph{Faster Convergence of MF-RAE over Baselines.} 
We compare the total convergence time of vanilla MF and our MF-RAE on ImageNet~256.
Vanilla MF is trained from scratch using the MFT (i.e., without a pre-trained teacher), requiring 1400 epochs, corresponding to about 7M iterations, with a total training cost of over 600 H100 GPU-days.

By contrast, MF-RAE proceeds in three stages: flow matching pre-training for 800 epochs / 1M iterations (78 H100 GPU-days), CMT mid-training for 27K iterations (2.1 H100 GPU-days), and MFD post-training for 36 epochs / 180K iterations (21 H100 GPU-days).
In total, MF-RAE requires only about 100 H100 GPU-days, representing more than a $6\times$ reduction in training cost compared to vanilla MF, while achieving faster convergence.

Moreover, once a flow matching teacher has been pre-trained or is available off the shelf, the additional cost of converting it into a few-step MF model via CMT mid-training and MF distillation is only 23 H100 GPU-days.
This shows that MF-RAE provides an efficient and practical route to distill a strong flow matching model into a fast few-step generator.

\paragraph{Ablation with Latent Representations and MFT versus MFD.}
\begin{table}[htbp!]
\vspace{-0.2cm}
\centering
\caption{Ablation on latent representation and training scheme for MF on ImageNet~256. Compared to MF trained on SD-VAE latents with DiT/SiT, our MF-RAE configuration
(RAE with DiT$^{\text{DH}}$ via MFD; last row) converges faster and achieves the best 1- and 2-step FIDs while requiring no guidance.
In contrast, MF trained on SD-VAE latents exhibits severe performance degradation when guidance is removed.
On a fixed latent space, MFD also outperforms MFT, showing that both the RAE representation and the MFD objective are key to fast and stable MF convergence.}
\label{tab:mf_original}
\fontsize{7.8pt}{8.8pt}\selectfont
\begin{tabular}{c@{\hskip 12pt}c@{\hskip 8pt}c@{\hskip 6pt}c@{\hskip 6pt}c}
\toprule
\textbf{Algorithm} & \textbf{Guided?}& \textbf{Architecture}  & \textbf{NFE} ($\downarrow$) & \textbf{FID} ($\downarrow$)  \\
\midrule
MFT  & $\checkmark$  & SD-VAE with DiT/SiT    & 1 / 2 & 3.38 / 2.20 \\ 
MFD  & $\checkmark$  & SD-VAE with DiT/SiT    & 1 / 2 & 3.15 / 1.95 \\ 
MFD  & $\times$  & SD-VAE with DiT/SiT    & 1 / 2 & 5.94 / 4.01 \\ 
MFT  & $\times$  & RAE with DiT$^\text{DH}$   & 1 / 2 & 2.81 / 2.56 \\ 
 \rowcolor{gray!20}
MFD  & $\times$   & RAE with DiT$^\text{DH}$      & 1 / 2 &   \textbf{2.03} / \textbf{1.89}    \\
\bottomrule
\end{tabular}
\end{table}

We empirically analyze how the choice of latent representation 
(SD-VAE paired with DiT/SiT vs. RAE paired with DiT$^{\text{DH}}$) 
and the training scheme (MFT vs. MFD) 
affect performance on ImageNet~256, with results summarized in \Cref{tab:mf_original}. The first row in \Cref{tab:mf_original} indicates the vanilla MF, while the last row is our MF-RAE.

For the SD-VAE setting, we use the vanilla MF-XL/2 configuration and a recent REG-based SiT teacher~\citep{wu2025reg}. 
This teacher attains FID $1.36$ with CFG and $1.80$ unguided, comparable to the RAE-space teacher (FID $1.51$ unguided). 
To ensure a fair comparison, we keep batch size, learning rate, EMA, CMT mid-training iterations, and optimizer identical across settings.

Comparing MF trained on SD-VAE latents (via MFD or MFT) with our MF-RAE
configuration (DiT$^{\text{DH}}$ + RAE via MFD), MF-RAE converges faster and
achieves the best FID among methods with similar teacher quality.
Moreover, we observe that MF on SD-VAE latents performs poorly without guidance, even with distillation, whereas MF-RAE attains high-quality unguided class-conditional generation. This shows that the RAE latents  is crucial for simplifying MF training and eliminating any guidance hyperparameters.

Interestingly, vanilla MF (trained on SD-VAE latents) can be trained from scratch with random initialization, but requires about $1400$ epochs (600+ H100 GPU-days) to converge. In contrast, MF on the semantic RAE latent space cannot be trained from scratch, either with random initialization or with a pre-trained diffusion teacher as initialization, unless we use CMT initialization. This indicates that the effectiveness of MF is strongly dependent on the choice of representation and architecture, and that our MF-RAE with CMT initialization provides a general stability mechanism that enables future model extensions.

Finally, we directly compare MFD and MFT on the same RAE latent representation (the last two rows in \Cref{tab:mf_original}).
In this setting, MFD with a pre-trained teacher achieves 1- and 2-step FIDs
of $2.03$ and $1.89$, while MFT with a one-point velocity reaches only $2.81$ and $2.56$, showing that MFD is substantially more effective.
Given a well-trained flow matching teacher, distillation supplies low-variance training and high-quality velocity targets (see Proposition~\ref{prop:bias-var}), so MF-RAE converges faster and to better
performance. Since the teacher-student FID gap is already small (the teacher with $50$ NFEs has FID $1.51$), an additional bootstrapping stage is unnecessary in this regime.

\subsection{Scale up to ImageNet~512 (\Cref{tab:imagenet-512})}

\begin{table}[htbp]
\centering
\caption{\small{Sample quality on class-conditional ImageNet~512. We report the sampling GFLOPS containing decoder and diffusion transformer costs.
Methods with * require additional complicated guidance hyperparameters obtained from extensive grid searches.}}
\label{tab:imagenet-512}
\fontsize{8.9pt}{9.9pt}\selectfont
\begin{tabular}{l@{\hskip 5pt}c@{\hskip 6pt}c@{\hskip 6pt}c@{\hskip 6pt}c}
\toprule
\textbf{METHOD} & \textbf{NFE} & \textbf{FID} ($\downarrow$)  &  \textbf{GFLOPS} ($\downarrow$) &\textbf{\#Params} \\
\midrule
\multicolumn{4}{l}{\textbf{Flow Map Models}} \\
\midrule 
CMT w/ ECD*~\citep{hu2025cmt} & 1 & 3.38 & 2344  & 1.5B \\
sCD*~\citep{lu2025sCM} & 1 & 2.28 & 2344  & 1.5B \\
sCT~\citep{lu2025sCM} & 1 & 4.29 & 2344  & 1.5B \\
AYF*~\citep{sabour2025align}& 1 & 3.32 & 1342  & 280M \\
 \rowcolor{gray!20}
MF-RAE~(Ours) & 1 & 3.23 & 1051  & 841M  \\
\bottomrule
\end{tabular}
\end{table}
We scale up our MF-RAE to ImageNet~512; to the best of our knowledge, this is among the first extensions of MF to this resolution. The sample quality (FID) and generation cost (GFLOPS) are summarized in \Cref{tab:imagenet-512}. Our method attains a competitive 1-step FID while achieving the lowest generation cost, owing to the efficient decoder. These near-SOTA results are obtained without any guidance: CMT and AYF distills from Auto-Guidance teachers, sCD distills from a CFG teacher, while sCT (which relies solely on one-point velocity estimation without a guided teacher) performs worse than ours. This highlights the simplicity and effectiveness of our approach and suggests that MF-RAE could be further improved by incorporating guidance techniques. Moreover, our results are obtained with substantially shorter training time than SOTA methods such as sCD and CMT.
More specifically, we perform MFD for 20K iterations, followed by MFT for an additional 10K iterations using the MFD checkpoint as initialization. The CMT stage requires approximately 8 H100 GPU days, while the combined MFD+MFT stage takes about 9 H100 GPU days. The total cost of 17 H100 GPU days is comparable to that of CMT with ECD (17 days) and significantly lower than sCD’s 233 days and sCT’s 98 days.
\paragraph{Ablation with Initializations Scheme and Bootstrapping Strategy.}
We empirically validate the proposed bootstrapping strategy, namely the combined MFD+MFT training scheme that replaces the pre-trained teacher with a one-point velocity objective for further fine-tuning.
We consider three initialization methods (random, flow matching, CMT) and three training algorithms (MFT only, MFD only, and the bootstrapped MFD+MFT combination), and train each configuration for 30K optimization iterations.

For initialization, both random and flow matching initializations lead to gradient explosions at the beginning of training, whereas only CMT yields stable optimization. 
This confirms the effectiveness of CMT as a trajectory-aware initialization for MF.

Among training schemes, MFT alone performs worst (one-step FID 5.82), consistent with its high gradient variance.
MFD alone converges faster due to lower variance, reaching near-convergence by 20K iterations (one-step FID 3.95).
Starting from that point, switching to the one-point MFT objective for an additional 10K iterations further reduces bias and attains one-step FID 3.23.
Thus, MFD is well-suited for the early stage to accelerate convergence, while a brief MFT phase refines the model, aligning with Proposition~\ref{prop:bias-var}.

\section{Conclusion}
Training MF in the RAE latent space is challenging.
By combining CMT for stable initialization, MFD for accelerated convergence, and an optional lightweight bootstrapping stage for further refinement, our approach
substantially reduces training cost, simplifies configuration by removing
guidance, and enables faster few-step generation while preserving
state-of-the-art performance.
Our pipeline provides a general recipe for training flow-map models in the RAE latent space and can readily incorporate future advances, demonstrating both flexibility and extensibility.

{
\clearpage
\newpage
    \small
    \bibliographystyle{ieeenat_fullname}
    \bibliography{main}
}

\newpage
\appendix
\onecolumn

\section{Theoretical Analysis of the Trade-off Between Pretrained and One-Point Velocities for Oracle MeanFlow Learning}\label{app:theory-proof}

\begin{proposition}
For any $\lambda \in [0,1]$, consider the combination of the one-point estimator
and the pre-trained teacher velocity
\[
\rvw_\lambda := (1-\lambda)\hat{\rvv} + \lambda\rvv_{\bphi}.
\]
Plugging $\rvw_\lambda$ into the target
$\rvh^{\mathrm{tgt}}_{\btheta^-}(\rvz_t,t,s;\rvw)$ (with $\rvw = \rvw_\lambda$)
yields the corresponding loss $\mathcal{L}_{\mathrm{MF}}(\btheta;\rvw_\lambda)$. Consider the following three residuals: the one-point velocity residual, the teacher–oracle velocity residual, and the oracle bias.
\begin{align*}
\delta\hat{\rvv}_t &:= \hat{\rvv}(\rvz_t,t) - \rvv(\rvz_t,t),\\
\delta\rvv_t^\bphi &:= \rvv_{\bphi}(\rvz_t,t) - \rvv(\rvz_t,t),\\
\delta\rvh(\rvz_t,t,s) &:= \rvh_{\btheta^-}(\rvz_t,t,s) - \rvh(\rvz_t,t,s).
\end{align*}
Then the MF loss admits the following decomposition:
\begin{align*}
\begin{aligned}
    \mathcal{L}_{\mathrm{MF}}(\btheta;\rvw_\lambda)
    = \mathbb{E}_{t,\rvz_t}
       \Big\|\rvh_{\btheta}
       - \big(\rvh + \rmB + \lambda\,\rmA_{\btheta^-}\,\delta\rvv_t^\bphi\big)\Big\|^2 
    + (1-\lambda)^2\,\mathbb{E}\big\|\rmA_{\btheta^-}\,\delta\hat{\rvv}_t\big\|^2,
\end{aligned}
\end{align*}
where $\rmB(\rvz_t,t,s)
:= (t-s)\Big(\partial_t \delta\rvh
  + \big(\nabla_{\rvx}\delta\rvh\big)\rvv\Big)$ and  $\rmA_{\btheta^-}(\rvz_t,t,s)
:= \rmI - (t-s)\,\nabla_{\rvx}\rvh_{\btheta^-}$.
\end{proposition}

\begin{proof}
Throughout the proof we fix $s$ and suppress the dependence on $(\rvz_t,t,s)$
whenever there is no ambiguity. All expectations are taken over $(t,\rvz_t)$ and when $\hat{\rvv}$ is random (e.g.\ a Monte--Carlo estimator).

We first derive the affine form of the teacher target around the oracle velocity
$\rvv$. Recall from the MF construction that for any proxy velocity $\rvw$,
the target associated with $\rvh_{\btheta^-}$ is
\[
\rvh_{\btheta^-}^{\mathrm{tgt}}(\rvz_t,t,s;\rvw)
=
\rvw - (t-s)\Bigl((\nabla_{\rvx}\rvh_{\btheta^-})\,\rvw
                  + \partial_t\rvh_{\btheta^-}\Bigr),
\]
so for two velocities $\rvw$ and $\rvv$ we have
\begin{align*}
\rvh_{\btheta^-}^{\mathrm{tgt}}(\rvz_t,t,s;\rvw)
 - \rvh_{\btheta^-}^{\mathrm{tgt}}(\rvz_t,t,s;\rvv)
&= \Bigl[\rvw - (t-s)(\nabla_{\rvx}\rvh_{\btheta^-})\,\rvw\Bigr]
 - \Bigl[\rvv - (t-s)(\nabla_{\rvx}\rvh_{\btheta^-})\,\rvv\Bigr] \\
&= \bigl(\rmI - (t-s)\,\nabla_{\rvx}\rvh_{\btheta^-}\bigr)\,(\rvw - \rvv).
\end{align*}
By the definition of $\rmA_{\btheta^-}$ in the proposition,
$\rmA_{\btheta^-} := \rmI - (t-s)\,\nabla_{\rvx}\rvh_{\btheta^-}$, so this can
be written as
\[
\rvh_{\btheta^-}^{\mathrm{tgt}}(\rvz_t,t,s;\rvw)
=
\rvh_{\btheta^-}^{\mathrm{tgt}}(\rvz_t,t,s;\rvv)
+ \rmA_{\btheta^-}(\rvz_t,t,s)\,\bigl(\rvw - \rvv(\rvz_t,t)\bigr).
\]
On the other hand, by the PDE relation between the oracle flow map $\rvh$ and
the oracle velocity $\rvv$, and by the definition
$\delta\rvh = \rvh_{\btheta^-} - \rvh$ and
$\rmB(\rvz_t,t,s)
:= (t-s)\bigl(\partial_t\delta\rvh + (\nabla_{\rvx}\delta\rvh)\,\rvv\bigr)$,
we have (see the derivation in the main text)
\[
\rvh_{\btheta^-}^{\mathrm{tgt}}(\rvz_t,t,s;\rvv)
= \rvh(\rvz_t,t,s) + \rmB(\rvz_t,t,s).
\]
Combining the last two displays gives, for every proxy velocity $\rvw$,
\begin{equation}
\label{eq:affine-target}
\rvh_{\btheta^-}^{\mathrm{tgt}}(\rvz_t,t,s;\rvw)
=
\rvh(\rvz_t,t,s)
+ \rmB(\rvz_t,t,s)
+ \rmA_{\btheta^-}(\rvz_t,t,s)\,\bigl(\rvw - \rvv(\rvz_t,t)\bigr),
\end{equation}
which is exactly the affine reparametrization of the target around the oracle
velocity $\rvv$.

We now prove the claimed decomposition of the MF loss.
By the definitions of the residual velocities,
\[
\hat{\rvv} = \rvv + \delta\hat{\rvv}_t,
\qquad
\rvv_{\bphi} = \rvv + \delta\rvv_t^\bphi,
\]
the mixed velocity satisfies, for any $\lambda\in[0,1]$,
\begin{align}
\label{eq:w-lambda-residual}
\rvw_\lambda
  = (1-\lambda)\hat{\rvv} + \lambda\rvv_{\bphi} 
  = (1-\lambda)\bigl(\rvv + \delta\hat{\rvv}_t\bigr)
     + \lambda\bigl(\rvv + \delta\rvv_t^\bphi\bigr) 
  = \rvv + (1-\lambda)\delta\hat{\rvv}_t + \lambda\delta\rvv_t^\bphi.
\end{align}
Substituting $\rvw = \rvw_\lambda$ and \Cref{eq:w-lambda-residual} into
\Cref{eq:affine-target} yields
\begin{align}
\rvh_{\btheta^-}^{\mathrm{tgt}}(\rvz_t,t,s;\rvw_\lambda)
&=
\rvh + \rmB
+ \rmA_{\btheta^-}\bigl(\rvw_\lambda - \rvv\bigr) \notag\\
&=
\rvh + \rmB
+ \rmA_{\btheta^-}\bigl((1-\lambda)\delta\hat{\rvv}_t
                       + \lambda\delta\rvv_t^\bphi\bigr) \notag\\
&=
\rvh(\rvz_t,t,s) + \rmB(\rvz_t,t,s)
+ \lambda\,\rmA_{\btheta^-}(\rvz_t,t,s)\,\delta\rvv_t^\bphi
+ (1-\lambda)\,\rmA_{\btheta^-}(\rvz_t,t,s)\,\delta\hat{\rvv}_t.
\label{eq:h-tgt-wlambda-simple}
\end{align}
For the remainder of the proof we fix $(t,\rvz_t)$ and abbreviate
\[
\rvh_{\btheta} := \rvh_{\btheta}(\rvz_t,t,s),
\quad
\rvh := \rvh(\rvz_t,t,s),
\quad
\rmA_{\btheta^-} := \rmA_{\btheta^-}(\rvz_t,t,s),
\quad
\rmB := \rmB(\rvz_t,t,s).
\]

By \Cref{eq:general_mf_loss}, the MF loss at $\rvw_\lambda$ is
\[
\mathcal{L}_{\mathrm{MF}}(\btheta;\rvw_\lambda)
  = \mathbb{E}\Bigl\|
      \rvh_{\btheta}(\rvz_t,t,s)
      - \rvh_{\btheta^-}^{\mathrm{tgt}}(\rvz_t,t,s;\rvw_\lambda)
    \Bigr\|_2^2.
\]
Using \Cref{eq:h-tgt-wlambda-simple}, we can write the pointwise residual as
\[
\rvh_{\btheta} - \rvh_{\btheta^-}^{\mathrm{tgt}}(\rvz_t,t,s;\rvw_\lambda)
=
\Bigl[\rvh_{\btheta}
      - \bigl(\rvh + \rmB + \lambda\,\rmA_{\btheta^-}\delta\rvv_t^\bphi\bigr)
\Bigr]
- (1-\lambda)\,\rmA_{\btheta^-}\delta\hat{\rvv}_t.
\]
Let
\[
\mathbf{Y}
:=
\rvh_{\btheta}
- \bigl(\rvh + \rmB + \lambda\,\rmA_{\btheta^-}\delta\rvv_t^\bphi\bigr),
\]
so that
\[
\rvh_{\btheta} - \rvh_{\btheta^-}^{\mathrm{tgt}}(\rvz_t,t,s;\rvw_\lambda)
= \mathbf{Y} - (1-\lambda)\,\rmA_{\btheta^-}\delta\hat{\rvv}_t.
\]
Therefore
\begin{align}
\mathcal{L}_{\mathrm{MF}}(\btheta;\rvw_\lambda)
=
\mathbb{E}\bigl\|
  \mathbf{Y} - (1-\lambda)\,\rmA_{\btheta^-}\delta\hat{\rvv}_t
\bigr\|_2^2 =
\mathbb{E}\|\mathbf{Y}\|_2^2
- 2(1-\lambda)\,\mathbb{E}\bigl\langle
    \mathbf{Y},\rmA_{\btheta^-}\delta\hat{\rvv}_t
  \bigr\rangle
+ (1-\lambda)^2\,\mathbb{E}\bigl\|
    \rmA_{\btheta^-}\delta\hat{\rvv}_t
  \bigr\|_2^2.
\label{eq:LMF-expanded-simple}
\end{align}

Since the one-point estimator is conditionally unbiased:
\[
\mathbb{E}\bigl[\hat{\rvv}(\rvz_t,t) \,\big|\, \rvz_t\bigr]
= \rvv(\rvz_t,t),
\qquad\text{so}\qquad
\mathbb{E}\bigl[\delta\hat{\rvv}_t \,\big|\, \rvz_t\bigr] = 0.
\]
For fixed $(t,\rvz_t)$, the quantities $\rvh_{\btheta}$, $\rvh$, $\rmB$,
$\rvv$, $\rvv_{\bphi}$, hence $\mathbf{Y}$ and $\rmA_{\btheta^-}$, are
deterministic, and the only randomness comes from the internal noise of
$\hat{\rvv}$. Thus
\begin{align*}
\mathbb{E}\bigl\langle
  \mathbf{Y},\rmA_{\btheta^-}\delta\hat{\rvv}_t
\bigr\rangle
&=
\mathbb{E}_{t,\rvz_t}\Bigl[
  \mathbb{E}\bigl[
    \langle \mathbf{Y},\rmA_{\btheta^-}\delta\hat{\rvv}_t\rangle
     \big| 
    \rvz_t
  \bigr]
\Bigr] \\
&=
\mathbb{E}_{t,\rvz_t}\Bigl[
  \bigl\langle
    \mathbf{Y}, 
    \rmA_{\btheta^-}\,\mathbb{E}[\delta\hat{\rvv}_t |\rvz_t]
  \bigr\rangle
\Bigr] \\
&=
\mathbb{E}_{t,\rvz_t}\bigl\langle
  \mathbf{Y},\rmA_{\btheta^-}\cdot \bm{0}
\bigr\rangle
= \bm{0}.
\end{align*}
Hence the cross term in \Cref{eq:LMF-expanded-simple} is zero, and we obtain
\begin{align*}
\mathcal{L}_{\mathrm{MF}}(\btheta;\rvw_\lambda)
=
\mathbb{E}_{t,\rvz_t}\Bigl\|
  \rvh_{\btheta}(\rvz_t,t,s)
  - \bigl(
      \rvh(\rvz_t,t,s)
      + \rmB(\rvz_t,t,s)
      + \lambda\,\rmA_{\btheta^-}(\rvz_t,t,s)\,\delta\rvv_t^\bphi
    \bigr)
\Bigr\|_2^2 
+ (1-\lambda)^2\,
  \mathbb{E}\bigl\|
    \rmA_{\btheta^-}(\rvz_t,t,s)\,\delta\hat{\rvv}_t
  \bigr\|_2^2,
\end{align*}
which is exactly the claimed decomposition \Cref{eq:LMF-mix}. This completes
the proof.
\end{proof}

From Proposition~\ref{prop:bias-var}, several implications follow. First, using the pre-trained velocity $\rvv_{\bphi}$ improves stability: when
$\lambda = 1$, the noisy one-point term vanishes, which reduces gradient
variance and typically stabilizes and accelerates optimization.

Second, this comes at the cost of bias.
The model no longer regresses to $\rvh$, but to
$\rvh + \rmB + \lambda\,\rmA_{\btheta^-}\,\delta\rvv_s^\bphi$.
When $\rvv_{\bphi}$ is accurate (small $\delta\rvv_s^\bphi$) and the time step
is small so that $\rmA_{\btheta^-} \approx \rmI$, this additional bias is small.
In the small-step regime, the shift is approximately
$\lambda\,\delta\rvv_s^\bphi$, so the practical target is close to the oracle
one whenever $\|\delta\rvv_s^\bphi\|$ is small.
However, under domain shift, $\delta\rvv_s^\bphi$ can be non-negligible and the
learned $\rvh_{\btheta}$ will partially compensate for this mismatch.

Third, the mixing weight $\lambda$ offers a natural tuning mechanism.
A schedule with $\lambda$ gradually increasing to $1$ can start from a regime
with low bias and low variance: initially, the one-point label maintains an
unbiased target while the pre-trained signal reduces variance; later, larger
$\lambda$ leverages the smoothness of $\rvv_{\bphi}$ for stable convergence. In practice, we use a simple two-stage schedule:
$\lambda = 0$ in the distillation stage and $\lambda = 1$ in the (optional) bootstrapping stage.

Finally, the coupling to the self-teacher remains simple.
Both $\rmB$ and $\rmA_{\btheta^-}$ are treated as stop-gradient with respect to
$\btheta^-$, so updates in $\btheta$ still solve a least-squares fit to a
shifted target.
As the self-teacher improves over training, $\delta\rvh$ shrinks and $\rmB$
decreases, further reducing the induced bias.

\section{More Hyperparameter Settings}\label{appendix:hyperparams}
\paragraph{MF-RAE's Flow Matching Pre-Training Setup.} We directly use the flow matching model checkpoint provided by RAE \cite{zheng2025rae}, where the model on ImageNet~256 is trained for 800 epochs, and that on ImageNet~512 is trained for 400 epochs.

\paragraph{MF-RAE's CMT Setup.} 
The CMT stage's learning rate is also 1e-4, but the EMA $\beta$ is set to 0.999 for faster convergence, thanks to CMT's stability and fixed target.


\textbf{MF-RAE's CMT Setup on ImageNet~256.} 
CMT generates eight trajectories from the teacher DiT$^\text{DH}$-XL without guidance during each iteration. We use the Euler ODE solver with 16 steps (FID=2.3). For each trajectory, it generates $16\times 15/2$ =120 pairs, but we randomly select 64 pairs for optimization to fit on the H100 GPU. The total batch size is thus $64\times 8=512$ pairs. We conduct this training on 8 H100 GPUs for 27K iterations. 


\textbf{MF-RAE's CMT Setup on ImageNet~512.} CMT generates two trajectories from the teacher DiT$^\text{DH}$-XL without guidance during each iteration. We use the Euler ODE solver with 16 steps (FID=1.66 without any guidance). For each trajectory, we randomly select 8 pairs for optimization to fit on the H100 GPU. The total batch size is thus $8\times2=16$ pairs. We accumulate gradients for eight rounds to enlarge the batch size. We conduct this training on 2 H100 GPUs for 27K iterations. 

\section{Generated Samples}
\begin{figure}[hbtp]
\centering
\includegraphics[width=0.89\linewidth]{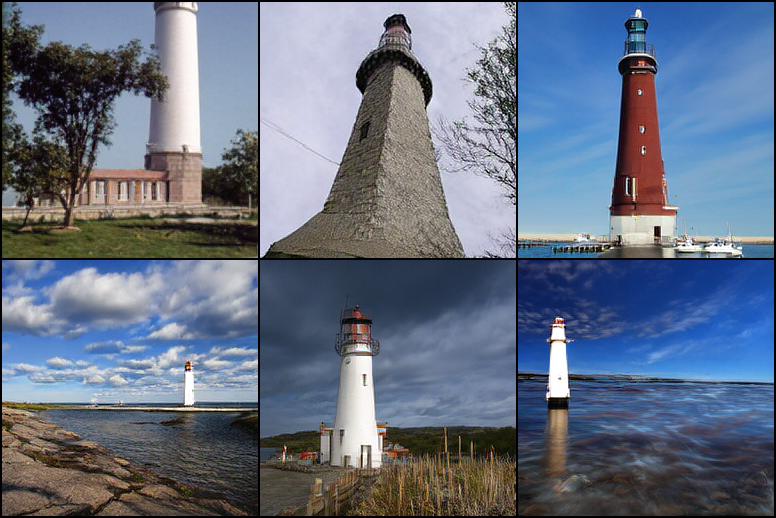}
\caption{ImageNet~256 MF-RAE 1-step samples on class 437: beacon, lighthouse, beacon light, pharos.}
\end{figure}
\begin{figure}[hbtp]
\centering
\includegraphics[width=0.89\linewidth]{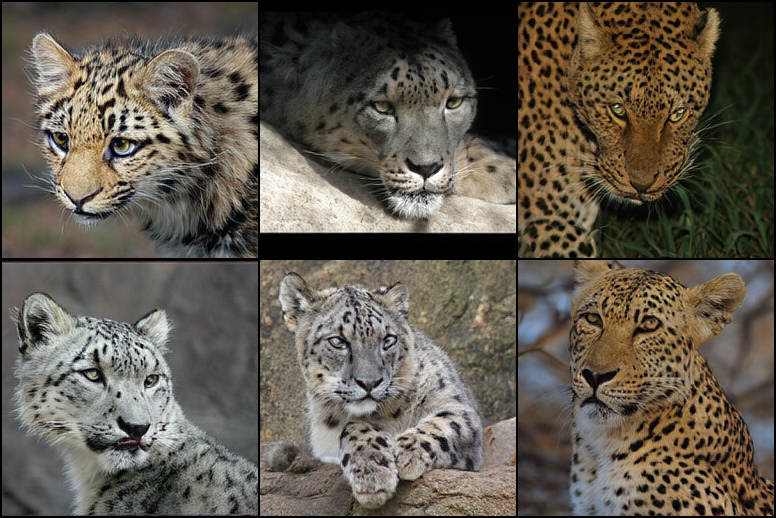}
\caption{ImageNet~256 MF-RAE 1-step samples on classes 288 and 290: leopard and snow leopard.}
\end{figure}
\begin{figure}[hbtp]
\centering
\includegraphics[width=0.89\linewidth]{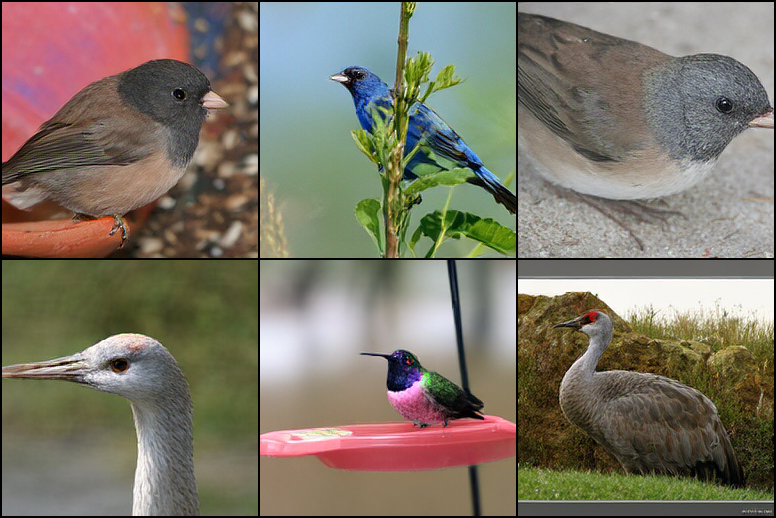}
\caption{ImageNet~256 MF-RAE 1-step samples on classes 13, 14, 94, and 134: snowbird, indigo bird, hummingbird, and crane bird.}
\end{figure}
\begin{figure}[hbtp]
\centering
\includegraphics[width=0.89\linewidth]{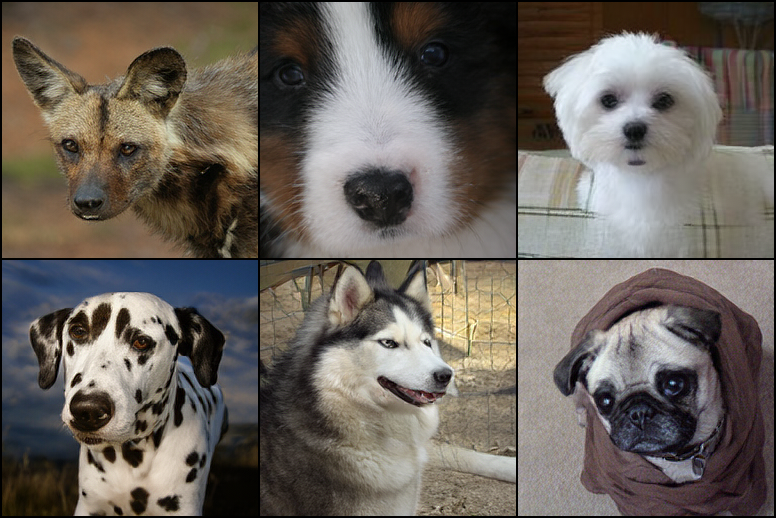}
\caption{ImageNet~256 MF-RAE 1-step samples for various dogs.}
\end{figure}
\begin{figure}[hbtp]
\centering
\includegraphics[width=0.89\linewidth]{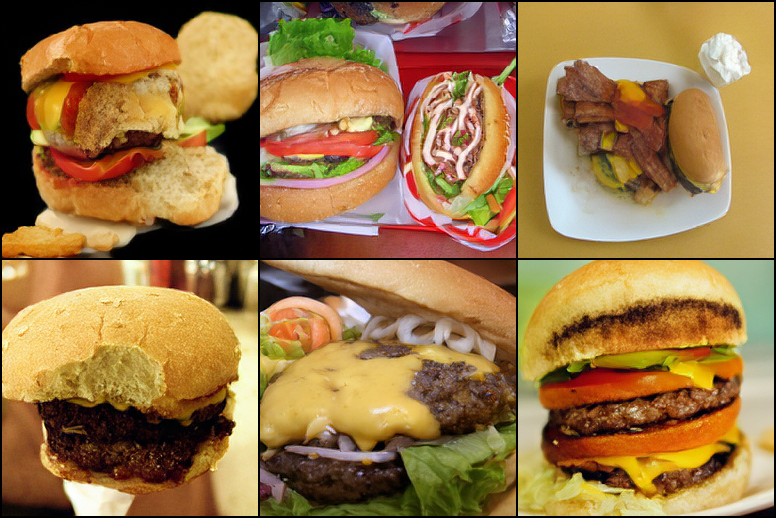}
\caption{ImageNet~256 MF-RAE 1-step samples for class 933: cheeseburger.}
\end{figure}
\begin{figure}[hbtp]
\centering
\includegraphics[width=0.89\linewidth]{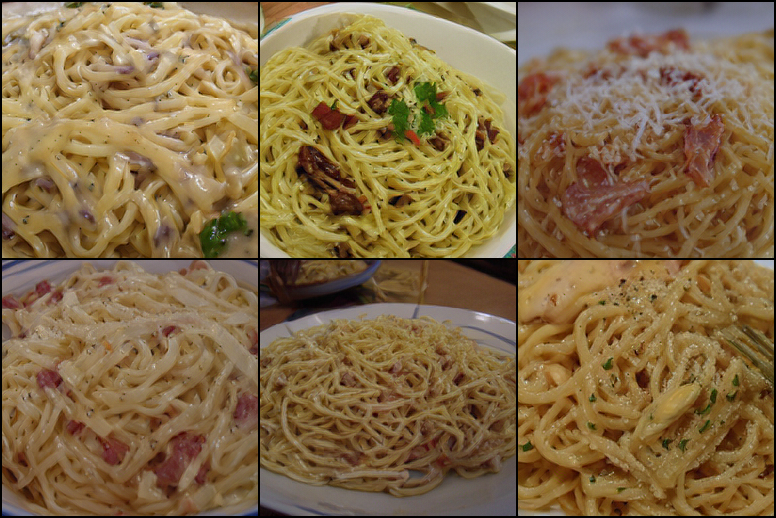}
\caption{ImageNet~256 MF-RAE 1-step samples for class 959: carbonara.}
\end{figure}
\begin{figure}[hbtp]
\centering
\includegraphics[width=0.89\linewidth]{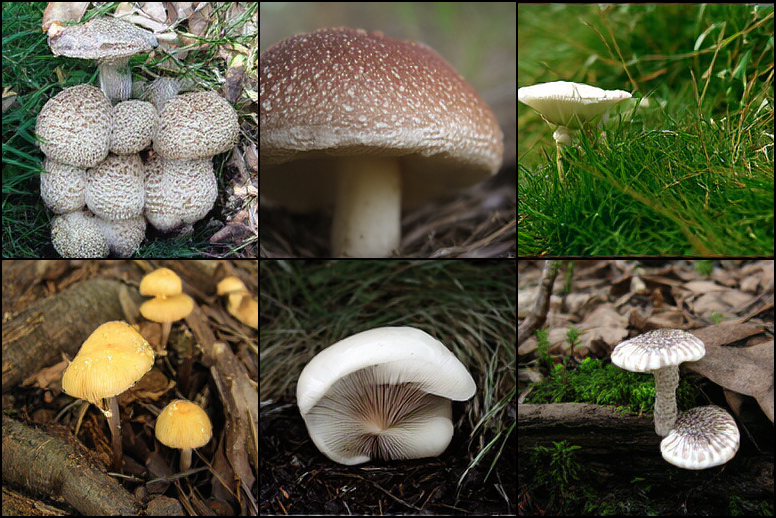}
\caption{ImageNet~256 MF-RAE 1-step samples for class 947: mushroom.}
\end{figure}
\end{document}